\def\year{2019}\relax
\documentclass[letterpaper]{article} 
\usepackage{aaai19}  
\usepackage{times}  
\usepackage{helvet}  
\usepackage{courier}  
\usepackage{url}  
\usepackage{graphicx}  
\usepackage{microtype}
\usepackage{subfigure}
\usepackage[utf8]{inputenc} 
\usepackage[T1]{fontenc}    
\usepackage{booktabs}       
\usepackage{amsfonts}       
\usepackage{nicefrac}       
\usepackage{amsmath}    
\usepackage{amssymb}    
\usepackage{multicol,lipsum}
\usepackage{multirow}
\usepackage{placeins}
\usepackage{tikz} 
\usetikzlibrary{arrows}
\usetikzlibrary{arrows.meta}
\usetikzlibrary{shapes}
\usetikzlibrary{decorations.pathreplacing}
\usepackage{xcolor}
\usepackage{color}
\usepackage{booktabs} 
\usepackage{lipsum}
\usepackage{mathtools}
\usepackage{cuted}
\usepackage{algorithm}
\usepackage{algorithmic}
\usepackage{amsthm}
\usepackage{commath}
\newtheorem{theorem}{Theorem}

\newtheorem{lemma}{Lemma}

\begin{document}
%
\title{Liquid Time-constant Recurrent Neural Networks as Universal Approximators}
\author{Ramin M. Hasani$^{1}$, Mathias Lechner$^{2}$, Alexander Amini$^{3}$, Daniela Rus,$^{3}$ and Radu Grosu$^{1}$\\
$^{1}$Cyber Physical Systems (CPS), Technische Universit\"{a}t Wien (TU Wien), 1040 Vienna, Austria\\
$^{2}$Institute of Science and Technology (IST), 3400 Klosterneuburg, Austria\\
$^{3}$ Computer Science and Artificial Intelligence Lab (CSAIL), Massachusetts Institute of Technology (MIT), Cambridge, USA\\
}
\maketitle

\begin{abstract}
In this paper, we introduce the notion of liquid time-constant (LTC) recurrent neural networks (RNN)s, a subclass of continuous-time RNNs, with varying neuronal time-constant realized by their nonlinear synaptic transmission model. This feature is inspired by the communication principles in the nervous system of small species. It enables the model to approximate continuous mapping with a small number of computational units. We show that any finite trajectory of an $n$-dimensional continuous dynamical system can be approximated by the internal state of the hidden units and $n$ output units of an LTC network. Here, we also theoretically find bounds on their neuronal states and varying time-constant. 
\end{abstract}

\section{Introduction}
Continuous-time spatiotemporal information processing can be performed by recurrent neural networks (RNN)s. In particular, a subset of RNNs whose hidden and output units are determined by ordinary differential equations (ODE), as in continuous-time recurrent neural networks (CTRNN)s \cite{funahashi1993approximation,mozer2017discrete}. Typically, in CTRNNs, the time-constant of the neurons' dynamics is a fixed constant value, and networks are wired by constant synaptic weights. We propose a new CTRNN model, inspired by the nervous system dynamics of small species, such as \emph{Ascaris} \cite{davis1989signaling}, \emph{Leech} \cite{lockery1992distributed}, and \emph{C. elegans} \cite{wicks1996dynamic,m2017sim}, in which synapses are nonlinear sigmoidal functions that model the biophysics of synaptic interactions. As a result, state of the postsynaptic neurons are defined by the incoming presynaptic nonlinearities to the cell. This attribute, originates varying time-constant for the cell and strengthen its individual neurons' expressivity in terms of output dynamics. 

Dynamic network simulations based on such models have been deployed in many application domains such as simulations of animals' locomotion \cite{wicks1996dynamic}, large-scale simulations of nervous systems \cite{gleeson2018c302,sarma2018openworm}, neuronal network's reachability analysis \cite{islam2016probabilistic}, model of learning mechanisms \cite{hasani2017non} and robotic control in reinforcement learning environments \cite{hasani2018re}.

In this paper, we formalize networks built based on such principles as liquid time-constant (LTC) RNNs (Sec. 2) and theoretically prove their universal approximation capabilities (Sec. 3). We also find bounds over their varying time-constant as well as their neuronal states (Sec. 4). 

\section{Liquid Time-constant RNNs}
Dynamics of a hidden or output \emph{neuron} $i$, $V_i(t)$, of an LTC RNN are modeled as a membrane integrator with the following ordinary differential equation (ODE) \cite{Koch98}:
\begin{equation}
\label{eqneuron}
C_{m_i} \frac{dV_i}{dt} = G_{Leak_i} \Big(V_{Leak_i} - V_i(t) \Big) + \sum_{j=1}^{n} I_{in}^{(ij)},
\end{equation}
with neuronal parameters: $C_{m_i}, G_{Leak_i}$ and $V_{Leak_i}$. $I_{in}^{(ij)}$ represents the external currents to the cell. Hidden nodes are allowed to have recurrent connections while they synapse into motor neurons in a feed-forward setting.  

\textit{Chemical synapses --} Chemical synaptic transmission from neuron $j$ to $i$, is modeled by a sigmoidal nonlinearity ($\mu_{ij}$,$\gamma_{ij}$), which is a function of the presynaptic membrane state, $V_j(t)$, and has maximum weight of $w_i$ \cite{Koch98}:

\begin{equation}
\label{eqchemsyn}
I_{s_{ij}} = \frac{w_{ij}}{1+ e^{-\gamma_{ij}(V_{j} + \mu_{ij})}} (E_{ij} - V_{i}(t)).
\end{equation}

The synaptic current, $I_{s_{ij}}$ is then linearly depends on the state of the neuron $i$. $E$, sets whether the synapse excites or inhibits the succeeding neuron's state. 

An \emph{electrical synapse} (\emph{gap-junction}), between node $j$ and $i$, was modeled as a bidirectional junction with weight, $\hat \omega_{ij}$, based on Ohm's law:
\begin{equation}
\label{eqgapjunc}
\hat{I}_{ij} = \hat \omega_{ij} \Big(v_j(t) - v_i(t)\Big).
\end{equation}

Internal state dynamics of neuron i, $V_i(t)$, of an LTC network, receiving one chemical synapse from neuron $j$, can be formulated as:
	\begin{equation}
	\label{eqoneneuron}
		 \frac{dV_i}{dt} = \frac{G_{Leak_i}}{C_{m_i}} \big( V_{Leak_i} - V_i(t)\big) + \frac{w_{ij}}{C_{m_i}} \sigma_i(V_j(t))(E_{ij} - V_{i}),
	\end{equation}
where $\sigma_i(V_j(t)) = 1/1+ e^{-\gamma_{ij}(V_{j} + \mu_{ij})} $. If we set the time-constant of the neuron $i$ as $\tau _i = \frac{C_{m_i}}{G_{Leak_i}}$, we can reform this equation as follows: 
	\begin{equation}
	\label{eqoneneuronreformed}
		 \frac{dV_i}{dt} = -\big(\frac{1}{\tau_i} +  \frac{w_{ij}}{C_{m_i}} \sigma_i(V_j)\big) V_i + \big(\frac{V_{leak_i}}{\tau_i}+ \frac{w_{ij}}{C_{m_i}} \sigma_i(V_j) E_{ij} \big). 
	\end{equation}
Eq. \ref{eqoneneuronreformed} presents an ODE system with a nonlinearly varying time-constant defined by $\tau_{system} = \frac{1}{1/\tau_i +  w_{ij}/C_{m_i} \sigma_i(V_j)}$, which distinguishes the dynamics of the LTC cells compared to the CTRNN cells. 

The overall network dynamics of the LTC RNNs with
$u(t)= [u_1(t), ..., u_{n+N}(t)]^T$ representing the internal states of $N$ interneurons (hidden units) and $n$ motor neurons (output units) can be written in matrix format as follows:
	\begin{equation}
	\label{eq1}
		\dot u(t) = - (1/\tau +W\sigma(u(t)))u(t) + A + W \sigma(u(t))B,
	\end{equation}
	in which $\sigma(x)$ is $C^1$-sigmoid functions and is applied element-wise. $\tau^{n+N}>0$ includes all neuronal time-constants, $A$ is an $n+N$ vector of resting states, $B$ depicts an $n+N$ vector of synaptic reversals, and $W$ is a $n+N$ vector produced by the matrix multiplication of a weight matrix of shape $(n+N) \times (n+N)$ and an $n+N$ vector containing the reversed value of all $C_{m_i}$s. Both $A$ and $B$ entries are bound to a range $[-\alpha, \beta]$ for $ 0 <\alpha < +\infty$, and $0 \leq \beta < +\infty$. $A$ contains all $V_{leak_i}/C_{m_i}$ and $B$ presents all $E_{ij}$s.

\section{Liquid time-constant RNNs are universal approximators}

In this section, we prove that any given finite trajectory of an $n$-dimensional dynamical system can be approximated by the internal and output states of an LTC RNN, with $n$ outputs, $N$ interneurons and a proper initial condition. Let $x=~[x_1,...,x_n]^T$ be the $n$-dimensional Euclidean space on  $\mathbb{R}^n$. 

\begin{theorem}
	Let $S$ be an open subset of $\mathbb{R}^n$ and $F:S\rightarrow\mathbb{R}^n$, be an autonomous ordinary differential equation, be a $C^1$-mapping, and $\dot x = F(x)$ determine a dynamical system on $S$. Let $D$ denote a compact subset of $S$ and we consider a finite trajectory of the system as: $I = [0,~T]$. Then, for a positive $\epsilon$, there exist an integer $N$ and a liquid time-constant recurrent neural network with $N$ hidden units, $n$ output units, such that for any given trajectory $\{x(t); t \in I\}$ of the system with initial value $x(0) \in D$, and a proper initial condition of the network, the statement below holds:
	\begin{center}
		$\underset{t \in I}{max} |x(t) - u(t)|< \epsilon$
	\end{center}	
\end{theorem}

We base our poof on the fundamental universal approximation theorem \cite{hornik1989multilayer} on feed-forward neural networks \cite{funahashi1989approximate,cybenko1989approximation,hornik1989multilayer}, recurrent neural networks (RNN) \cite{funahashi1989approximate,schafer2006recurrent} and time-continuous RNNs \cite{funahashi1993approximation}. We first define Lemma \ref{lem5} to be used in the proof of Theorem 1. 

\begin{lemma}
\label{lem5}
	for an $F: \mathbb{R}^n \rightarrow \mathbb{R^+}^n$ which is a bounded $C^1$-mapping, the differential equation
	\begin{equation}
	\label{eq7}
		\dot x = - (1/\tau +F(x))x + A + B F(x),
	\end{equation}
	in which $\tau$ is a positive constant, and A and B are constants coefficients bound to a range $[-\alpha, \beta]$ for $ 0 <\alpha < +\infty$, and $0 \leq \beta < +\infty$,
	has a unique solution on $[0, \infty)$.
\end{lemma}
\begin{proof}
	Based on the assumptions, we can take a positive $M$, such that 
	\begin{equation}
		0 \leq F_i(x) \leq M (\forall i = 1,...,n)
	\end{equation}
	by looking at the solutions of the following differential equation: 
	\begin{equation}
		\dot y = - (1/\tau +M)y + A + B M,
	\end{equation}	
we can show that
\begin{equation}
\tiny
	min\{|x_i(0)|, \frac{\tau(A+BM)}{1+\tau M}\} \leq x_i(t) \leq max\{|x_i(0)|, \frac{\tau(A+BM)}{1+\tau M}\}, 
\end{equation}
if we set the output of the max to $C_{max_i}$ and the output of the min to $C_{min_i}$ and also set $C_1 = min\{C_{min_i}\}$ and $C_2 = max\{C_{max_i}\}$, then the solution $x(t)$ satisfies
\begin{equation}
	\sqrt{n}C_1 \leq x(t) \leq \sqrt{n}C_2.
\end{equation}
Based on Lemma 2 and Lemma 3 in \cite{funahashi1993approximation}, a unique solution exists on the interval $[0,+\infty)$.	
\end{proof}

Lemma \ref{lem5} demonstrates that an LTC network defined by Eq. \ref{eq7}, has a unique solution on $[0,\infty)$, since the output function is bound and $C^1$. 

\subsubsection{Proof of Theorem 1}
\begin{proof}
For proving Theorem 1, we adopt similar steps to that of Funahashi and Nakamura on the approximation ability of continuous time RNNs \cite{funahashi1993approximation}, to approximate a dynamical system with a larger dynamical system given by an LTC RNN. 

Part 1. We choose an $\eta$ which is in range $(0, min\{\epsilon, \lambda\})$, for $\epsilon > 0$, and $\lambda$ the distance between $\tilde D$ and boundary $\delta S$ of $S$. $D_{\eta}$ is set:
\begin{equation}
	D_{\eta} = \{ x \in \mathbb{R}^n; \exists z \in \tilde D, |x-z| \leq \eta \}.
\end{equation}
$D_{\eta}$ stands for a compact subset of $S$, because $\tilde D$ is compact. Thus, $F$ is Lipschitz on $D_{\eta}$ by Lemma 1 in \cite{funahashi1993approximation}. Let $L_F$ be the Lipschitz constant of $F | K_{\eta}$, then, we can choose an $\epsilon _l > 0$, such that
\begin{equation}
	\epsilon _l < \frac{\eta L_F}{2(exp L_F T-1)}.
\end{equation}

Based on the universal approximation theorem, there is an integer $N$, and an $n \times N$ matrix $B$, and an $N \times n$ matrix $C$ and an $N$-dimensional vector $\mu$ such that
\begin{equation}
\label{eq26}
	max |F(x) - B \sigma(Cx + \mu)| < \frac{\epsilon_l}{2}.
\end{equation}

We define a $C^1$-mapping $\tilde F: \mathbb{R}^n \rightarrow \mathbb{R}^n$ as: 
\begin{equation}
\label{eq27}
\small
	\tilde F(x) = - (1/\tau +W_l \sigma(Cx+\mu))x + A + W_l B \sigma(Cx + \mu),
\end{equation}
with parameters matching that of Eq. \ref{eq1} with $W_l =W$. 

We set the system's time constant, $\tau_{sys}$ to:

\begin{equation}
	\tau_{sys} = \frac{1}{1/\tau +W_l \sigma(Cx+\mu)}. 
\end{equation}
We chose a large $\tau_{sys}$, conditioned with the following:

\begin{align}
&(a)~~\forall x \in D_\eta;~~\abs{\frac{x}{\tau _{sys}}} < \frac{\epsilon _l}{2}\\
&(b)~~\abs{\frac{\mu}{\tau _{sys}}}< \frac{\eta L_{\tilde G}}{2(exp L_{\tilde G} T - 1)}~\text{and}~\abs{\frac{1}{\tau _{sys}}} < \frac{L_{\tilde G}}{2},
\end{align}

where $L_{\tilde G}/2$ is a lipschitz constant for the mapping $W_l\sigma: \mathbb{R}^{n+N} \rightarrow \mathbb{R}^{n+N}$ which we will determine later. To satisfy conditions (a) and (b), $\tau W_l << 1$ should hold true.

Then by Eq. \ref{eq26} and \ref{eq27}, we can prove:
\begin{equation}
	\underset{x \in D_\eta}{max} \abs{F(x) - \tilde F(x)}< \epsilon _l
\end{equation}

Let's set $x(t)$ and $\tilde x(t)$ with initial state $x(0) = \tilde x(0) = x_0 \in D$, as the solutions of equations below:
\begin{equation}
	\dot x = F(x),
\end{equation}
\begin{equation}
	\dot {\tilde x} = \tilde F(x).
\end{equation}

Based on Lemma 5 in \cite{funahashi1993approximation}, for any $t \in I$,

\begin{align}
\abs{x(t) - \tilde x(t)} &\leq \frac{\epsilon _l}{L_F}(exp L_F t -1) \\
&\leq \frac{\epsilon _l}{L_F}(exp L_F T -1).
\end{align}

Thus, based on the conditions on $\epsilon$,

\begin{equation}
\label{eq36}
	\underset{t \in I}{max} \abs{x(t) - \tilde x(t)}< \frac{\eta}{2}.
\end{equation}

Part 2. Let's Considering the following dynamical system defined by $\tilde F$ in Part 1:

\begin{equation}
\label{eq37}
	\dot{\tilde x} = -\frac{1}{\tau_{sys}} \tilde x + A_1 + W_lB \sigma(C\tilde x + \mu).
\end{equation}

Suppose we set $\tilde y = C\tilde x + \mu$; then:
\begin{equation}
	\dot{\tilde y} = C \dot{\tilde x} = -\frac{1}{\tau_{sys}} \tilde y + E \sigma(\tilde y) + A_2 + \frac{\mu}{\tau_{sys}},
\end{equation}
where $E = CW_lB$, an $N \times N$ matrix. We define

\begin{equation}
	\tilde z = [\tilde x_1, ..., \tilde x_n, \tilde y_1,...,\tilde y_n]^T, 
\end{equation}
and we set a mapping $\tilde G: \mathbb{R}^{n+N} \rightarrow \mathbb{R}^{n+N}$ as:
\begin{equation}
\label{eq40}
		\tilde G(\tilde z) = -\frac{1}{\tau_{sys}} \tilde z + W \sigma(\tilde z) + A + \frac{\mu _1}{\tau_{sys}},
\end{equation}

\begin{align}
	&W^{(n+N)\times(n+N)} = \left(\begin{array}{cc} 0 & B\\ 0 & E \end{array}\right), \\
	&\mu _1^{n+N} = \left(\begin{array}{c} 0 \\ {\mu} \end{array}\right),~~~ A^{n+N} = \left(\begin{array}{c} {A_1} \\ {A_2} \end{array}\right). 
\end{align}

By using Lemma 2 in \cite{funahashi1993approximation}, we can show that solutions of the following dynamical system:
\begin{equation}
	\dot{\tilde z} = \tilde G(\tilde z),~~~~\tilde y(0) = C\tilde x(0) + \mu, 
\end{equation}
are equivalent to the solutions of the Eq. \ref{eq37}. 

Let's define a new dynamical system $G: \mathbb{R}^{n+N} \rightarrow \mathbb{R}^{n+N}$ as follows:
\begin{equation}
\label{eq44}
	G(z) = -\frac{1}{\tau_{sys}} z + W \sigma(z) + A,
\end{equation}
where $z = [x_1, ..., x_n, y_1,...,y_n]^T$. Then the dynamical system below
\begin{equation}
\label{eqzzz}
	\dot z = -\frac{1}{\tau_{sys}} z + W \sigma(z) + A,
\end{equation}
can be realized by an LTC RNN, if we set $h(t) = [h_1(t),...,h_N(t)]^T$ as the hidden states, and $u(t) = [U_1(t),...,U_n(t)]^T$ as the output states of the system. Since $\tilde G$ and $G$ are both $C^1$-mapping and $\sigma^{\prime}(x)$ is bound, therefore, the mapping $\tilde z \rightarrow W\sigma(\tilde z) +A$ is Lipschitz on $\mathbb{R}^{n+N}$, with a Lipschitz constant $L_{\tilde G}/2$. As $L_{\tilde G}/2$ is Lipschitz constant for $- \tilde z/\tau_{sys}$ by condition (b) on $\tau_sys$, $L_{\tilde G}$ is a Lipschitz constant of  $\tilde G$. 

From Eq. \ref{eq40}, Eq. \ref{eq44}, and condition (b) of $\tau_{sys}$, we can derive the following:
\begin{equation}
	\abs{\tilde G(z) - G(z)} = \abs{\frac{\mu}{\tau_{sys}}} < \frac{\eta L_{\tilde G}}{2(exp L_{\tilde G}T-1)}.
\end{equation}
Accordingly, we can set $\tilde z(t)$ and $z(t)$, solutions of the dynamical systems:
\begin{equation}
	\dot{\tilde z} = \tilde G(z),~~~ \begin{cases} \tilde x(0) = x_0 \in D \\ \tilde y(0) = Cx_0 +\mu \end{cases}
\end{equation} 
\begin{equation}
	\dot{z} = G(z),~~~ \begin{cases} u(0) = x_0 \in D \\ \tilde h(0) = Cx_0 +\mu \end{cases}
\end{equation} 
By Lemma 5 of \cite{funahashi1993approximation}, we achieve
\begin{equation}
	\underset{t \in I}{max} \abs{\tilde z(t) - z(t)}< \frac{\eta}{2},
\end{equation}
and therefore we have:
\begin{equation}
\label{eq50}
	\underset{t \in I}{max} \abs{\tilde x(t) - u(t)}< \frac{\eta}{2},
\end{equation}

Part3. Now by using Eq. \ref{eq36} and Eq. \ref{eq50}, for a positive $\epsilon$, we can design an LTC network with internal dynamical state $z(t)$, with $\tau_{sys}$ and $W$. For x(t) satisfying $\dot x = F(x)$, if we initialize the network by $u(0) = x(0)$ and $h(0) = Cx(0) +\mu$, we obtain:
\begin{equation}
	\underset{t \in I}{max} \abs{x(t) - u(t)}< \frac{\eta}{2}+\frac{\eta}{2}= \eta < \epsilon.
\end{equation}
\end{proof}

REMARKS. The LTC's network architecture allows interneurons (hidden layer) to have recurrent connections to each other, however it assumes a feed forward connection stream from hidden nodes to the motor neuron units (output units). We assumed no inputs to the system and principally showed that the interneurons' network together with motor neurons can approximate any finite trajectory of an autonomous dynamical system. The proof subjected an LTC RNN with only chemical synapses. It is easy to extend the proof for a network which includes gap junctions as well, since their contribution to the network dynamics is by adding a linear term to the time-constant of the system ($\tau_{sys}$), and to the equilibrium state of a neuron, $A$ in Eq \ref{eqzzz}. 

\section{Bounds on $\tau_{sys}$ and state of an LTC RNN}
In this section, we prove that the time-constant and the state of neuronal activities in an LTC RNN is bound to a finite range, as depicted in lemmas 2 and 3, respectively.
\begin{lemma}
Let $v_i$ denote the state of a neuron $i$, receiving $N$ synaptic connections of the form Eq. \ref{eqchemsyn}, and P gap junctions of the form Eq. \ref{eqgapjunc} from the other neurons of a LTC network $G$, if dynamics of each neuron's state is determined by Eq. \ref{eqneuron}, then the time constant of the activity of the neuron, $\tau_i$, is bound to a range:
\begin{equation}
\label{eq:tau}
	C_i/(g_i+\sum^N_{j=1}w_{ij}+ \sum^P_{j=1}\hat w_{ij}) \leq \tau_i \leq C_i/(g_i + \sum^P_{j=1} \hat w_{ij}),
\end{equation}
\end{lemma}

\begin{proof}
	The sigmoidal nonlinearity in Eq. \ref{eqchemsyn}, is a monotonically increasing function, bound to a range $0$ and $1$:
	\begin{equation}
			 0 < S(Y_j, \sigma _{ij}, \mu _{ij}, E_{ij}) < 1
	\end{equation}	
By replacing the upper-bound of $S$, in Eq. \ref{eqchemsyn} and then substituting the synaptic current in Eq. \ref{eqneuron}, we will have: 
\begin{equation}
\label{eq:neuron1}
	C_i \frac{dv_i}{dt} = g_i. (V_{leak} - v_i)~+ \sum_{j=1}^N w_{ij}(E_{ij} - v_i) + \sum_{j=1}^P \hat w_{ij}(v_j - v_i),	
\end{equation} 

\begin{align}
	C_i \frac{dv_i}{dt} = \underbrace{(g_i V_{leak}~+~\sum_{j=1}^N w_{ij}E_{ij})+\sum_{j=1}^P \hat w_{ij}v_j}_\text{A} \\
	- \underbrace{(g_i ~+~\sum_{j=1}^N w_{ij} + \sum_{j=1}^P \hat w_{ij})}_\text{B} v_i,
\label{eq:neuron2}
\end{align}	

\begin{equation}
\label{eq:neuron3}
	C_i \frac{dv_i}{dt} = A - B v_i.
\end{equation}
By assuming a fixed $v_j$, Eq. \ref{eq:neuron3} is an ordinary differential equation with solution of the form: 
\begin{equation}
\label{eq:neuron33}
	v_i(t) = k_1 e^{-\frac{B}{C_i}t} + \frac{A}{B}.
\end{equation}
From this solution, one can derive the lower bound of the system's time constant, $\tau _i^{min}$: 
\begin{equation}
\tau _i^{min} = \frac{C_i}{B} = \frac{C_i}{g_i ~+~\sum_{j=1}^N w_{ij} + \sum_{j=1}^P \hat w_{ij}}. 
\end{equation}

By replacing the lower-bound of $S$, in Eq. \ref{eq:neuron1}, the term $\sum_{j=1}^N w_{ij}(E_{ij} - v_i)$ becomes zero, therefore:

\begin{equation}
	C_i \frac{dv_i}{dt} = \underbrace{(g_i V_{leak}~+~\sum_{j=1}^P \hat w_{ij}v_j)}_\text{A}
	- \underbrace{(g_i ~+~\sum_{j=1}^P \hat w_{ij})}_\text{B} v_i.
\label{eq:neuron4}
\end{equation}

Thus, we can derive the upper-bound of the time constant, $\tau _i^{max}$:
\begin{equation}
\tau_i^{max} = \frac{C_i}{g_i~+~\sum_{j=1}^P \hat w_{ij}}.
\end{equation}
\end{proof}

\begin{lemma}
Let $v_i$ denote the state of a neuron $i$, receiving $N$ synaptic connections of form Eq. \ref{eqchemsyn}, from the other nodes of a network $G$, if dynamics of each neuron is determined by Eq. \ref{eqneuron}, then the hidden state of the neurons on a finite trajectory, $I=[0,T](0<T<+\infty)$, is bound as follows:
\begin{equation}
\label{eq:vbound}
	\underset{t \in I}{min}(V_{leak_i}, E_{ij}^{min}) \leq v_i(t) \leq \underset{t \in I}{max}(V_{leak_i}, E_{ij}^{max}),
\end{equation}
\end{lemma}

\begin{proof}
Let us insert $ M = max\{V_{leak_i}, E_{ij}^{max}\}$ as the membrane potential $v_i(t)$ into Eq. \ref{eq:neuron1}:
\begin{equation}
\label{eq:neuronmax}
	C_i \frac{dv_i}{dt} = \underbrace{g_i (V_{leak}-M)}_\text{$\leq 0$} + \underbrace{\sum_{j=1}^N w_{ij}\sigma(v_j)(E_{ij}-M)}_\text{$\leq 0$}.
\end{equation}

Right hand side of Eq. \ref{eq:neuronmax}, is negative based on the conditions on M, positive weights and conductances, and the fact that $\sigma(v_i)$ is also positive in $\mathbb{R}^N$. Therefore, the left hand-side must also be negative and if we conduct an approximation on the derivative term:
\begin{equation}
	C_i \frac{dv_i}{dt} \leq 0,~~~ 
	\frac{dv_i}{dt} \approx \frac{v(t+\delta t) - v(t)}{\delta t} \leq 0,
\end{equation}

holds. by Substituting $v(t)$ with $M$, we have the following: 
\begin{equation}
	\frac{v(t+\delta t) - M}{\delta t} \leq 0~\rightarrow~v(t+\delta t) \leq M
\end{equation}

and therefore:
\begin{equation}
	v_i(t) \leq \underset{t \in I}{max}(V_{leak_i}, E_{ij}^{max}).
\end{equation}

Now if we substitute the membrane potential, $V_(i)$ with $ m = min\{V_{leak_i}, E_{ij}^{min}\}$, following the same methodology used for the proof of the upper bound, we can derive
\begin{equation}
	\frac{v(t+\delta t) - m}{\delta t} \leq 0~\rightarrow~v(t+\delta t) \leq m,
\end{equation}

and therefore:

\begin{equation}
	v_i(t) \geq \underset{t \in I}{min}(V_{leak_i}, E_{ij}^{min}).
\end{equation}	 
\end{proof}

\section{Conclusions}
We proved the universal approximation capability of liquid time-constant (LTC) RNNs, and showed how their varying dynamics are bound in a finite range. We believe that our work builds up the preliminary theoretical bases for investigating the capabilities of LTC networks.

\section*{Acknowledgments}
R.M.H. and R.G. are partially supported by Horizon-2020 ECSEL Project grant No. 783163 (iDev40), and the Austrian Research Promotion Agency (FFG), Project No. 860424. A.A. is supported by the National Science Foundation (NSF) Graduate Research Fellowship Program.

\bibliographystyle{aaai}

\end{document}